\documentclass[10pt,twocolumn,letterpaper]{article}

\usepackage{iccv}
\makeatletter
\@namedef{ver@everyshi.sty}{}
\makeatother
\usepackage{times}
\usepackage{epsfig}
\usepackage{graphicx}
\usepackage{amsmath}
\usepackage{amssymb}

\usepackage{amsthm}
\usepackage{pdfpages}
\usepackage{pgffor}
\makeatletter
\AtBeginDocument{\let\LS@rot\@undefined}
\makeatother

\newtheorem{theorem}{Theorem}

\newtheorem{definition}{Definition}
\newtheorem{lemma}[theorem]{Lemma}

\newcommand{\R}{\mathbb{R}}

\newcommand{\T}{\mathbb{T}}
\newcommand{\F}{\mathcal{F}}
\newcommand{\pipe}{:}
\newcommand{\NN}{{N\times N}}
\newcommand{\TZ}{T_A(N)}

\usepackage{multirow}

\def\supplementfilename{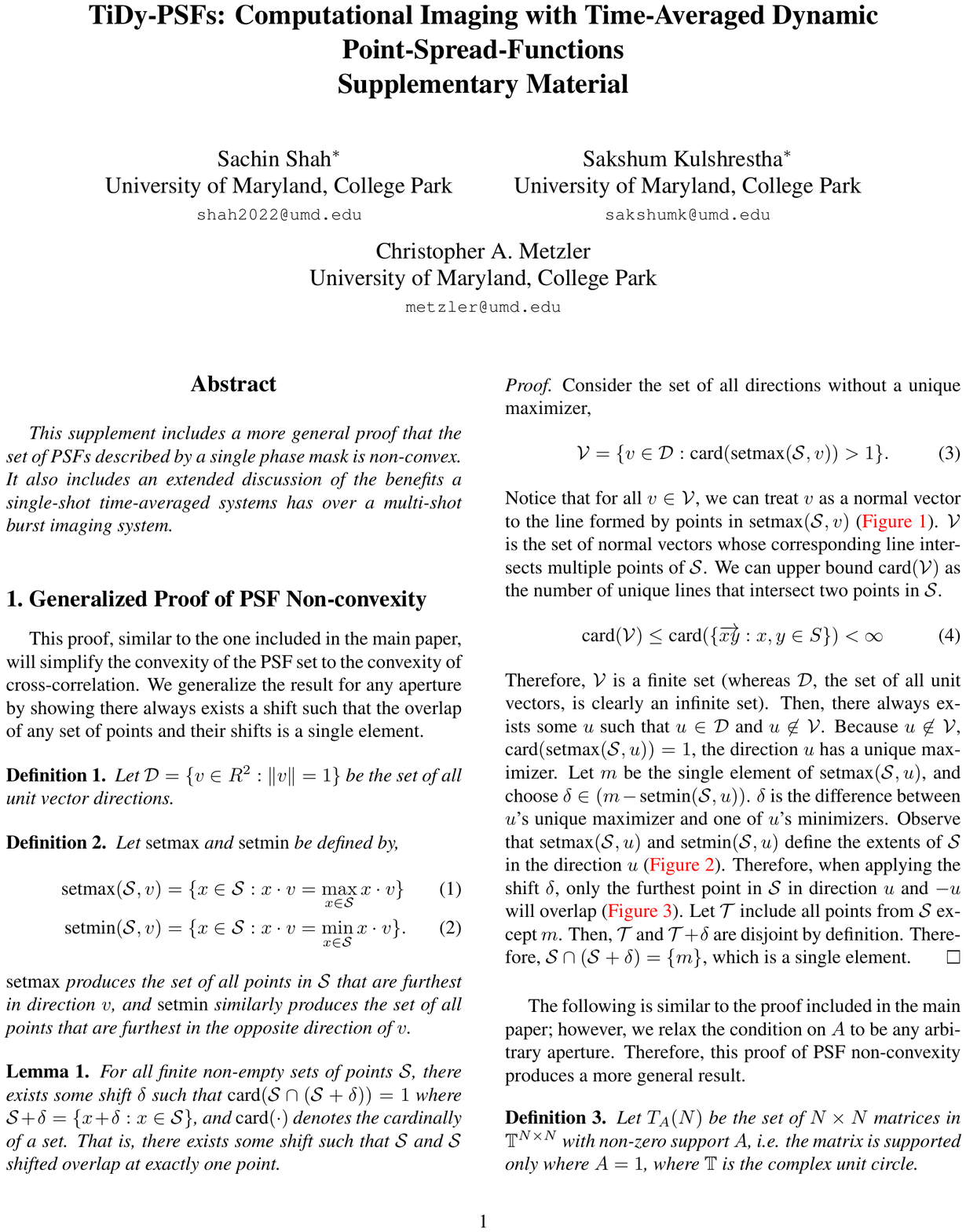}
\pdfximage{\supplementfilename}
\def\numbersupplementpages{\the\pdflastximagepages}

\usepackage[pagebackref=true,breaklinks=true,letterpaper=true,colorlinks,bookmarks=false]{hyperref}

\iccvfinalcopy

\begin{document}

\title{TiDy-PSFs: Computational Imaging with Time-Averaged Dynamic Point-Spread-Functions}

\author{Sachin Shah$^*$\\
University of Maryland, College Park\\
{\tt\small shah2022@umd.edu}
\and
Sakshum Kulshrestha$^*$\\
University of Maryland, College Park\\
{\tt\small sakshumk@umd.edu}
\and Christopher A. Metzler\\
University of Maryland, College Park\\
{\tt\small metzler@umd.edu}
}

\maketitle
\def\thefootnote{*}\footnotetext{These authors contributed equally to this work}\def\thefootnote{\arabic{footnote}}

\thispagestyle{empty}

\begin{abstract}
Point-spread-function (PSF) engineering is a powerful computational imaging techniques wherein a custom phase mask is integrated into an optical system to encode additional information into captured images. 
Used in combination with deep learning, such systems now offer state-of-the-art performance at monocular depth estimation, extended depth-of-field imaging, lensless imaging, and other tasks. 
 Inspired by recent advances in spatial light modulator (SLM) technology, this paper answers a natural question: Can one encode additional information and achieve superior performance by changing a phase mask dynamically over time? 
We first prove that the set of PSFs described by static phase masks is non-convex and that, as a result, time-averaged PSFs generated by dynamic phase masks are { fundamentally more expressive}. 
We then demonstrate, in simulation, that time-averaged dynamic (TiDy) phase masks can offer substantially improved monocular depth estimation and extended depth-of-field imaging performance.
\end{abstract}


\section{Introduction}
Extracting depth information from an image is a critical task across a range of applications including autonomous driving \cite{wang2019pseudo, AutoDriving}, robotics \cite{Sabnis:2011, Ye2017SelfSupervisedSL}, microscopy \cite{Fischer:2011, Palmieri:2019}, and augmented reality \cite{Woo2011DepthassistedR3, Lu:2019}.
To this end, researchers have developed engineered phase masks and apertures which serve to encode depth information into an image~\cite{CodedAperture, tetrapod}. 
To optimize these phase masks, recent works have exploited deep learning: By simultaneously optimizing a phase mask and a reconstruction algorithm ``end-to-end learning'' is able to dramatically improve system performance~\cite{phasecam, Sitzmann:2018}. 

\begin{figure}[t]
\begin{center}
   \includegraphics[width=0.95\linewidth]{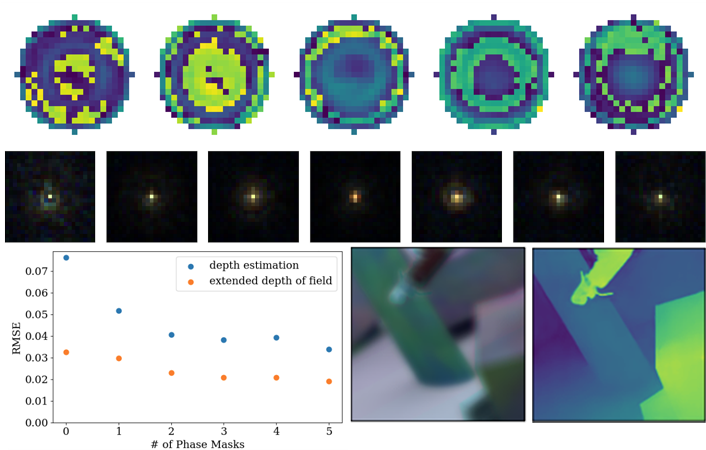}
\end{center}
   \caption{{\bf Time-averaged Dynamic PSFs} Top: Phase mask sequence that was optimized to perform simultaneous extended depth-of-field imaging and monocular depth estimation. Middle: Proposed TiDy PSFs at specific depths. Bottom left: Depth estimation and all-in-focus imaging performance improve as one averages over more phase masks. Bottom right: Depth-encoded image and reconstructed depth map.}
   \vspace{-5pt}
\label{fig:teaser}
\end{figure}

Most existing works have focused on learning or optimizing a single phase mask for passive depth perception. 
We conjecture that this restriction leaves much room for improvement. Perhaps by using an SLM to introduce a sequence of phase masks over time, one could do much better.

Supporting this idea is the fact, which we prove in \autoref{thm:psf-convex}, that the set of PSFs described by a single phase mask is non-convex. 
This implies that time-averaged PSFs, which span the convex hull of this set, can be significantly more expressive. 
In this work, we exploit the PSF non-convexity by developing a multi-phase mask end-to-end optimization approach for learning a sequence of phase masks whose PSFs are averaged over time.

This work's central contributions are as follows:
\begin{itemize}
    \item We prove the set of PSFs generated by a single phase mask, is non-convex.
    Thus, dynamic phase-masks offer a fundamentally larger design space.
    \item We extend the end-to-end learning optics and algorithm design framework to design a dynamic set of phase masks. 
    \item We demonstrate, in simulation, that time-averaged PSFs can achieve superior monocular depth estimation and extended depth-of-field imaging performance. 
\end{itemize}

\section{Background}
\paragraph{Image Formation Model.}
One can simulate the formation of an an image in a camera by discretizing an RGB image by depth, convolving each depth with it's the corresponding PSF, and compositing the outputs to form the signal on the sensor. This process can be represented by the equation
\begin{align}\label{eq:forward_model}
    I=\sum_{d=1}^D O_d \left( L * h_d  \right) ,
\end{align}
where $L$ represents all-in-focus image, $\{1, \cdots, D\}$ represent a set of discrete depth layers, $O_d$ is the occlusion mask at depth $d$, and the set $\{h_1, \cdots, h_D\}$ represent the depth-dependent PSF, i.e.,~the cameras response to point sources at various depths \cite{Hasinoff07}. Other works assume no depth discontinuities \cite{Sitzmann:2018} or add additional computation to improve blurring at depth boundaries \cite{Ikoma:2021}. Our model is similar to those used in \cite{phasecam, Chang:2019:DeepOptics3D}. 

\paragraph{PSF Formation Model.}
A PSF $h_d$ can be formed as a function of distance $d$ and phase modulation $\phi^{M}$ caused by height variation on a phase mask.
\begin{align}
    h_d=|\F[ A \exp( i \phi^{DF}(d) + i \phi^{M}  ) ]|^2
\end{align}
where $\phi^{DF}(d)$ is the defocus aberration due to the distance $d$ between the focus point and the depth plane. Note that because this PSF depends on depth, it can be used to encode depth information into $I$ \cite{Goodman:2017}.  \\

The key idea behind PSF-engineering and end-to-end learning is that one can use the aforementioned relationships to encode additional information into a captured image $I$ by selecting a particularly effective mask $\phi^M$.

\section{Related Work}

\subsection{Computational Optics for Depth Tasks}

Optics based approaches for depth estimation use sensors
and optical setups to encode and recover depth information.
Modern methods have used the depth-dependent blur caused by an aperture to estimate the depth of pixels in
an image. These approaches compare the
blur at different ranges to the expected blur caused by an aperture focused at
a fixed distance \cite{DepthFromDefocus}. Groups improved on this idea by implementing coded apertures, retaining more high frequency information about the scene to disambiguate depths \cite{CodedAperture}.
Similar to depth estimation tasks, static phase masks have been used to produce tailored PSFs more invariant to depth, allowing for extended depth-of-field imaging \cite{Dowski:1995}. However, these optically driven approaches have been passed in performance by modern deep neural networks, allowing for joint optimization of optical elements and neural reconstruction networks.

\subsection{Deep Optics}
Many methods have engineered phase masks with specific depth qualities. By maximizing Fisher information for depth, the coded image theoretically will have the most amount of depth cues as possible \cite{Fisher} and by minimizing Fisher information, one may achieve an extended depth-of-field image \cite{Dowski:1995}. Deep learning techniques can be used to jointly train the optical parameters and neural network based estimation methods. The idea is that one can ``code'' an image to retain additional information about a scene, and then use a deep neural network to produce reconstructions. By using a differentiable model for light propagation, back-propagation can be used to update phase mask values simultaneously with neural network parameters. This approach was demonstrated for extended depth-of-field imaging \cite{Sitzmann:2018, Ikoma:2021, Liu:22}, depth estimation \cite{phasecam, Chang:2019:DeepOptics3D, Ikoma:2021}, and holography \cite{Choi:2021, choi2022time}. While these previous approaches successfully improved performance, they focused on enhancing a single phase mask. We build on these works by simultaneously optimizing multiple phase masks, which allows us to search over a larger space of PSFs.

\section{Theory}
\label{sec:Thm}
Micro-ElectroMechanical SLMs offer high framerates but have limited phase precision due to heavy quantization \cite{Bartlett:2019}. As \cite{choi2022time} noted, intensity averaging of multiple frames can improve quality by increasing effective precision to overcome quantization. Our key insight is that even as SLM technology improves, intensity averaging yields a more expressive design space than a single phase mask. This is supported by the claim that the set of PSFs that can be generated by a single phase mask is non-convex. We provide a rigorous proof for the claim as follows. 

\begin{definition}\label{def:A}
$A \in \{0,1\}^\NN$ is some valid aperture with a non-zero region $S$ such that there exists lines $L_1$ and $L_2$ where $S$ can be contained between them, and $L_1 \parallel L_2$ and $u = S\cap L_1$ and $v = S\cap L_2$ are single points (\autoref{fig:aperture_shape}).
\end{definition}
This definition of $A$ supports most commonly used apertures including but not limited to circles, squares, and $n$-sided regular polygons. See supplement for proof for all shapes.
\begin{definition}
    Let $\TZ$ be the set of $\NN$ matrices in $\T^\NN$ with non-zero support $A$, i.e. the matrix is supported only where $A=1$, where $\T$ is the complex unit circle.
\end{definition}
The PSF induced by a phase mask $M$ can be modeled as the squared magnitude of the Fourier transform of the pupil function $f$ \cite{phasecam}.
\begin{definition}
    Let $f:\R^\NN\to T_A(N)$ be defined by
    \begin{align}f(M)=A\odot\exp(iD+icM)\end{align}
    where $\odot$ denotes entry-wise multiplication, and $D\in\R^\NN$ and $c\in\R-\{0\}$ (the reals except for $0$) are fixed constants. 
\end{definition}
\begin{definition}
    Let $g:\TZ\to\R^\NN$ be defined by
    \begin{align}
     g(X)=\frac{|\F(X)|\odot|\F(X)|}{\|\F(X)\|_{F}^2}
    \end{align}
    where $\F$ denotes the discrete Fourier Transform with sufficient zero-padding, $|\cdot|$ denotes entry-wise absolute value, and $\|\cdot\|_{F}$ denotes the Frobenius norm.
\end{definition}
\begin{lemma}
From fourier optics theory \cite{Goodman:2017}, any single phase mask's PSF at a specific depth can be written as
\[PSF = g\circ f.\]
\end{lemma}
\begin{figure}[t]
    \centering
    \includegraphics[width=0.7\linewidth]{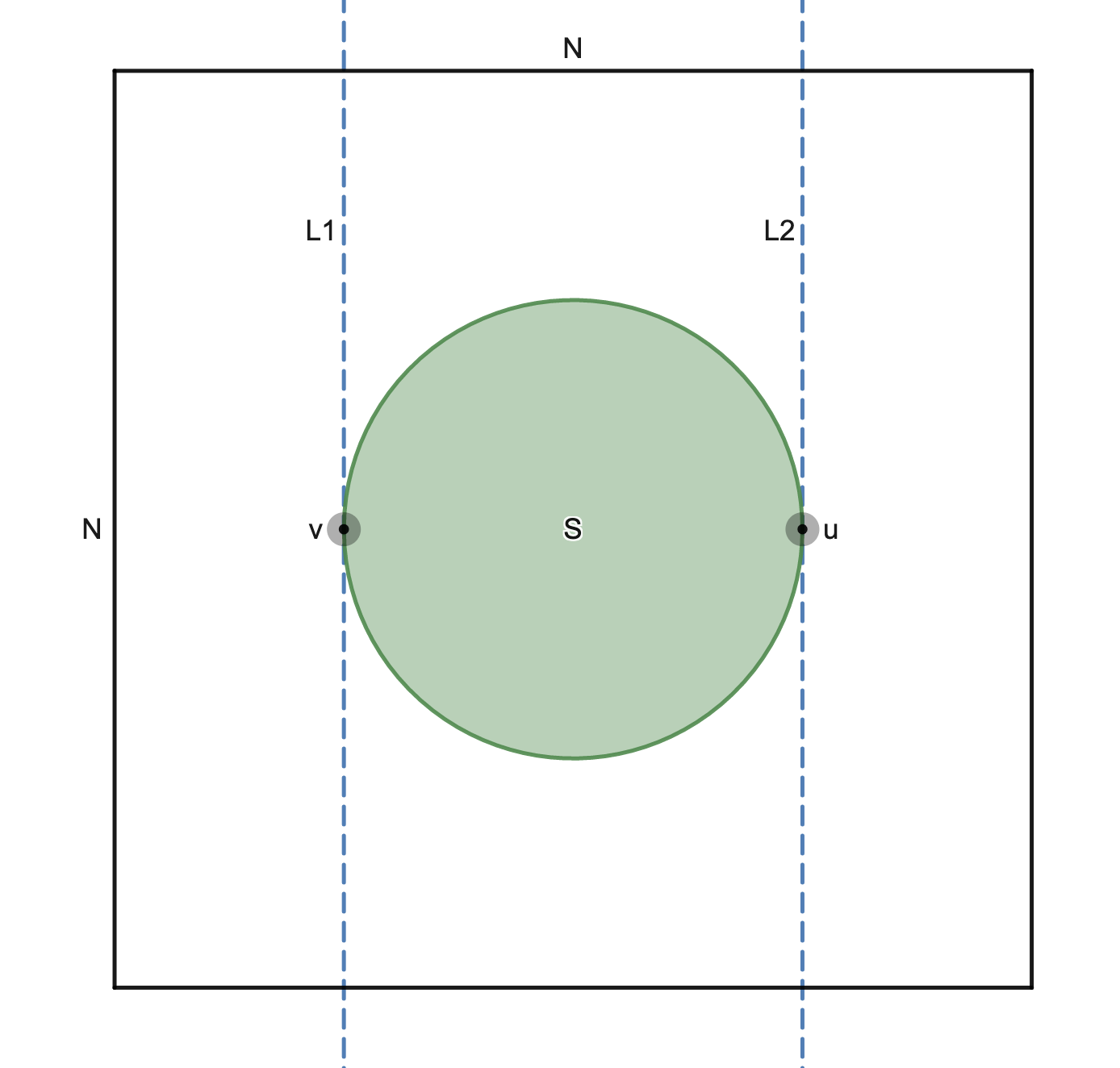}
    \caption{\textbf{Example aperture that satisfies constraints on $\mathbf{A}$.} The aperture is fitted between parallel lines $L1$ and $L2$, which only intersect the aperture at one point each. Common aperture shapes fit into these constraints. 
    }
    \label{fig:aperture_shape}
\end{figure}

\begin{theorem}\label{thm:psf-convex}
    The range of PSF is not a convex set.
\end{theorem}
\begin{proof}
$f$ is clearly surjective, so it suffices to argue the range of $g$ is not convex. Assume by way of contradiction that the range of $g$ is convex. Then, for all $X^{(1)} ,\ldots, X^{(k)} \in \TZ$ there exists $Y \in \TZ$ such that $g(Y) = \frac{1}{k}\sum_{i=1}^k g(X^{(i)})$.
By Parseval's Theorem,
\begin{align}
    \|\F(X)\|_{F}^2 = N^2 \|X\|_F^2 = N^2 \sum_{i=0}^N\sum_{j=0}^N A_{i,j}
\end{align}
so the condition is
\begin{align}
    |\F(Y)|\odot |\F(Y)| = \frac{1}{k}\sum_{i=1}^k |\F(X^{(i)})|\odot |\F(X^{(i)})|
\end{align}
or equivalently
\begin{align}
    \F(Y)\odot\overline{\F(Y)} = \frac{1}{k} \sum_{i=1}^k \F(X^{(i)})\odot\overline{\F(X^{(i)})}.
\end{align}
Then the cross-correlation theorem reduces it to
\begin{align}
    \F(Y \star Y) = \frac1k \sum_{i=1}^k \F(X^{(i)} \star X^{(i)})
\end{align}
where $\star$ denotes cross-correlation. Because the Fourier Transform is linear we finally have
\begin{align}
    Y \star Y = \frac1k \sum_{i=1}^k X^{(i)}\star {X^{(i)}}.
\end{align}
Therefore, the convexity of the range of $g$ is equivalent to the convexity of the set $\{X\star X\pipe X\in \TZ\}$. We will show the set's projection onto a particular coordinate is not convex.
\begin{align}
(X\star X)_{s,r} =  \sum_{i=0}^N\sum_{j=0}^N X_{i,j} \overline{X_{i+s,j+r}}
\end{align}
where we adopt the convention that $X_{s,r}=0$ when $s,r>N$ or $s,r<0$. Take the points $u$ and $v$ from the definition of $A$ (\autoref{def:A}). Also observe that correlation can be represented geometrically as shifting $\overline{X}$ over $X$. In this representation, notice that as the shift $(s,r)$ approaches $v-u$, the non-zero overlap between $X$ and $\overline{X}$ shifted by $(s,r)$ approaches $1$ by construction. That is, when $L_1$ is shifted to overlap $L_2$, $u$ and $v$ will be the only non-zero overlaps between the shifted and original non-zero points (\autoref{fig:shift}). No other non-zero points can overlap above or below $L_2$ by definition of $S$. Therefore, $(X\star X)_{v-u}$ becomes
\begin{align}
X_u\overline{X_v} + \sum_{i=1}^{N^2-1} 0.
\end{align}
Because $X_u\overline{X_v} \in \T$, $(X\star X)_{v-u} \in\T$ which is a non-convex set. Therefore, the set of correlation's of values on the complex unit circle masked by $A$ is also not convex, and so is $PSF$.
\begin{figure}[t]
    \centering
    \includegraphics[width=0.7\linewidth]{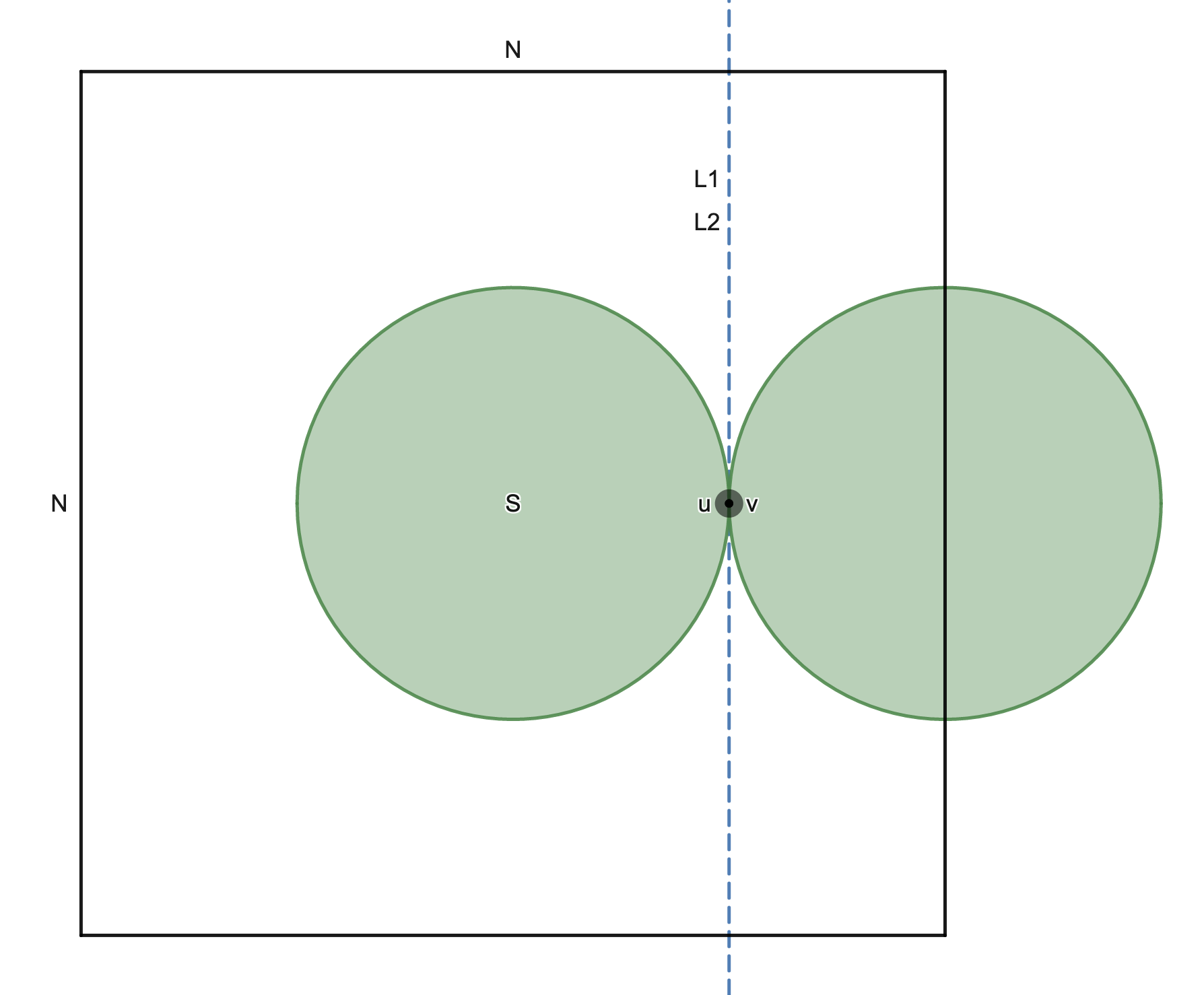}
    \caption{\textbf{Geometric interpretation of correlation $\mathbf{(X\star X)_{v-u}}$.} The figure represents the correlation step when the shift is $v-u$. Notice that only $u$ and $v$ overlap once the shift is applied.}
    \label{fig:shift}
\end{figure}
\end{proof}

Time-averaged PSFs span the convex hull of the set of static-mask PSFs, meaning there exists some PSFs achievable only through intensity averaging PSFs from a sequence of phase masks. This implies multi-phase mask learning may reach a better minimum.

\section{Multi-Phase Mask Optimization}
\begin{figure*}
\begin{center}
\includegraphics[width=0.9\linewidth]{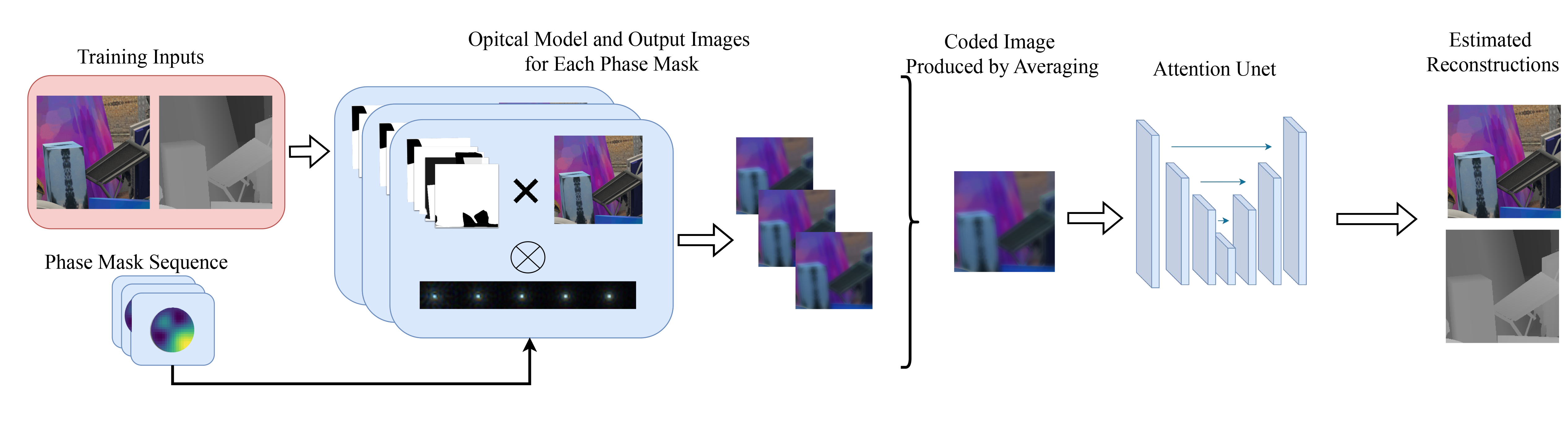}
\end{center}
   \caption{ \textbf{Multi-phase mask forward model overview.} A sequence of phase masks are used to generate a sequence of depth-dependent PSFs. These PSFs are convolved with depth masked clean images to simulate depth dependent convolution. The images produced by each phase mask are averaged to create a coded image which is fed into an attention U-Net. The reconstruction loss is back-propagated end-to-end through the network and the optical model to design phase masks and algorithms capable of performing monocular depth estimation and extended depth-of-field simultaneously.}
   \vspace{-15pt}
\label{fig:optical-model}
\end{figure*}

\subsection{Optical Forward Model}

Similar to PhaseCam3D \cite{phasecam}, we model light propagation using Fourier optics theory \cite{Goodman:2017}. In contrast to previous work, we compute the forward model \eqref{eq:forward_model} for multiple phase masks, producing a stack of output images, which when averaged form our coded image. This coded image simulates the recorded signal from imaging a scene using a sequence of phase masks in a single exposure (\autoref{fig:optical-model}).

\subsection{Specialized Networks}
For the monocular depth estimation task, we use the MiDaS Small network \cite{midas}. This is a well known convolutional monocular depth estimation network designed to take in natural images and output relative depth maps. The network is trained end-to-end with the phase masks. A mean-squared error (MSE) loss term is defined in terms of the depth reconstruction prediction, $\hat{D}$ and the ground truth depth map $D$,
\begin{align}
        L_{Depth} = \frac{1}{N} \lVert D - \hat{D} \rVert_2^2
\end{align}
where $N$ is the number of pixels.
This process allows for the simultaneous optimization of the phase masks as well as fine tuning MiDaS to reconstruct from our coded images. 

For the extended depth-of-field task, we use an Attention U-Net \cite{oktay2018attention} to reconstruct all-in-focus images. The network is optimized jointly with the phase mask sequence. To learn a reconstruction $\hat{I}$ to be similar to the all-in-focus ground truth image $I$, we define the loss term using MSE error
    \begin{align}
        L_{AiF} = \frac{1}{N} \lVert I - \hat{I} \rVert_2^2
    \end{align}
    where $N$ is the number of pixels.

\subsection{Joint Task Optimization}

We also present an alternative to the specialized networks: a single network jointly trained for monocular depth estimation and extended depth-of-field using a sequence of phase masks. This network has a basic Attention U-Net architecture outputting $4$ channels representing depth maps as well as all-in-focus images. Similar to prior works, we use a combined loss function, adding a coefficient to weight the losses for each individual task:
\begin{align}
    L_{total} = \lambda_{Depth}L_{Depth} + \lambda_{AiF}L_{AiF}.
\end{align}

\section{Experiments}

\subsection{Training Details}

We use the FlyingThings3D from Scene Flow Datasets~\cite{flyingthings}, which uses synthetic data generation to obtain all-in-focus RGB images and disparity maps. We use the cropped $278\times 278$ all-in-focus images from \cite{phasecam}. In total, we use $5077$ training patches and $419$ test patches.

Both the optical layer and reconstruction networks are differentiable, so the phase mask sequence and neural network can be optimized through back-propagation. Each part is implemented in PyTorch. During training, we use the Adam \cite{ADAM} optimizer with parameters $\beta_1=0.99$ and $\beta_2=0.999$. The learning rate for the phase masks is $10^{-8}$ and for the reconstruction network it is $10^{-4}$, and the batch size was $32$. Finally, training and testing were performed on NVIDIA Quadro P6000 GPUs.

We parameterize $23\times 23$ phase masks pixel-wise as \cite{Liu:22} found pixel-wise parameterization to produce the best overall performance. The monocular depth estimation task uses a the MiDaS Small architecture pretrained weights for monocular depth estimation downloadable from PyTorch \cite{midas}. The extended depth-of-field task pretrains an Attention U-Net with a fixed Fresnel lens for $300$ epochs. For the joint task, we set $\lambda_{Depth}=\lambda_{AiF}=1$ to balance overall performance, and we pretrain the Attention U-Net for $300$ epochs with a fixed Fresnel lens.
In simulation, the red, blue, and green channels are approximated by discretized wavelengths, $610$ nm, $530$ nm, and $470$ nm respectively. Additionally, the depth range is discretized into $21$ bins on the interval $[-20, 20]$, which is larger than previous works.

\subsection{Evaluation Details}
For ablation studies on our method, we used the testing split of the FlyingThings3D set for both monocular depth estimation and extended depth-of-field imaging \cite{flyingthings}. For comparisons to existing work, we also tested our monocular depth estimation network on the labeled NYU Depth v2 set \cite{NYU}. The ground truth depth maps were translated to layered masks for the clean images by bucketing the depth values into $21$ bins, allowing us to convolve each depth in an image with the required PSF. We use root mean squared error (RMSE) between ground truth and estimated depth maps for depth estimation and RMSE between ground truth and reconstructed all-in-focus images for extended depth-of-field imaging. We also use peak signal-to-noise ratio (PSNR) and structural similarity index~\cite{ssim} (SSIM) for extended depth-of-field imaging.

\subsection{Ablation Studies}

\subsubsection{Effect of Phase Mask Sequence Length}
\begin{figure}[t]
\begin{center}
\includegraphics[width=0.8\linewidth]{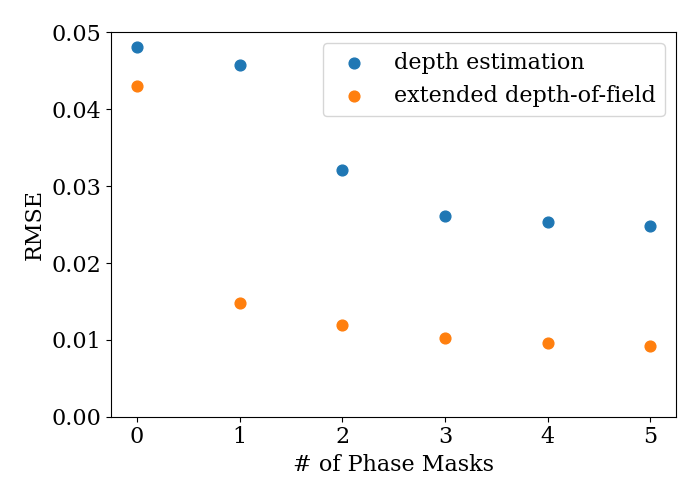}
\end{center}
   \caption{ \textbf{RMSE for specialized tasks} for each phase mask sequence length. RMSE decreases with respect to phase mask sequence length for both specialized extended depth-of-field imaging and monocular depth estimation tasks. $0$ phase masks refers to a reconstruction neural network with a fixed Fresnel lens.}
\label{fig:both-num-masks}
\end{figure}

\begin{figure}[t]
\begin{center}
\includegraphics[width=0.8\linewidth]{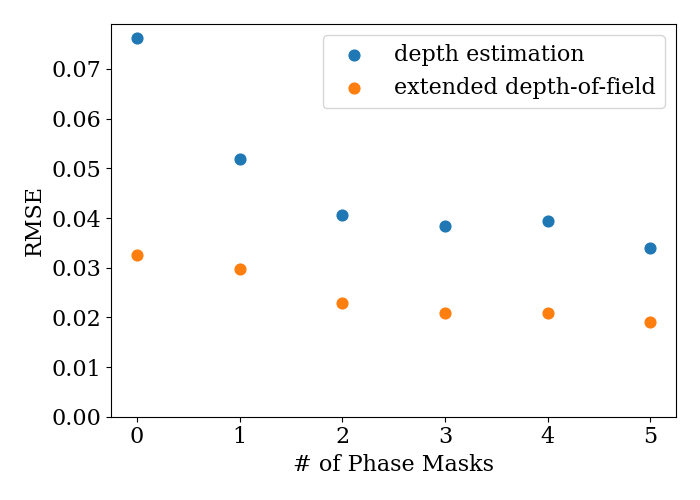}
\end{center}
   \caption{ \textbf{RMSE for joint optimization of monocular depth estimation and extended depth-of-field imaging} for each phase mask sequence length. RMSE decreases with respect to phase mask sequence length for this complex joint task, demonstrating the benefit of multi-phase mask learning. $0$ phase masks refers to a reconstruction neural network with a fixed Fresnel lens.}
\label{fig:combo-num-masks}
\end{figure}

For both all-in-focus imaging and depth estimation, we vary the phase mask count that the end-to-end system is trained with to gauge the benefits of using multiple phase masks. The forward model and initial phase masks were held standard while the phase mask count was varied. The resulting networks were evaluated at convergence. For the extended depth-of-field task, the masks were all initialized with random noise uniform from $0$ to $1.2 \times 10^{-6}$. For the depth estimation task, the masks were initialized with the Fisher mask with added Gaussian noise parameterized by a $5.35\times 10^{-7}$ mean and $3.05\times 10^{-7}$ standard deviation. 

End-to-end optimization on each task with a specialized network yielded improved performance as the phase mask count increased, visualized in \autoref{fig:both-num-masks}. This result implies that sequences of phase masks are successful in making the PSF space more expressive. Additionally, even for the more complex joint task, learning a system that can produce both all-in-focus images and depth maps, error decreases with phase mask count until a plateau, visualized in \autoref{fig:combo-num-masks}.

\subsubsection{All-in-focus without Reconstruction Networks}

\begin{figure}[t]
\begin{center}
\includegraphics[width=0.9\linewidth]{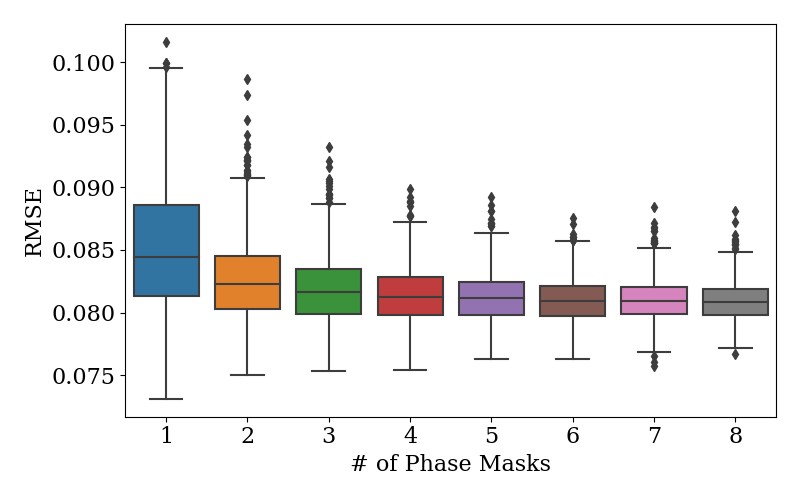}
\end{center}
   \caption{ \textbf{All-in-focus imaging RMSE distribution} for each phase mask length without a reconstruction network.  The best RMSE for each phase mask count has low correlation with respect to phase mask sequence length, but the variance of RMSE decreases.}
\label{fig:edof-no-recon}
\end{figure}

A phase mask generating a PSF of the unit impulse function at every depth would be ideal for extended depth-of-field as each depth is in focus. If possible, this phase mask would not require any digital processing. We optimize phase mask sequences of varying lengths to produce an averaged PSF close to the unit impulse function for all depths. For each sequence length, phase masks are optimized using MSE loss between the unit impulse function and the averaged PSF at each depth until convergence. We ran $1000$ trials of random phase mask initialization for each length. Observe that a side-effect of longer phase masks is training stability. The range of RMSE between the simulated capture image and ground truth all-in-focus image decreases as the sequence length increases (\autoref{fig:edof-no-recon}). This indicates training longer sequences is more resilient to initialization. 

\subsubsection{Phase Mask Initialization for Depth Perception}

\begin{figure}[t]
\begin{center}
\includegraphics[width=0.8\linewidth]{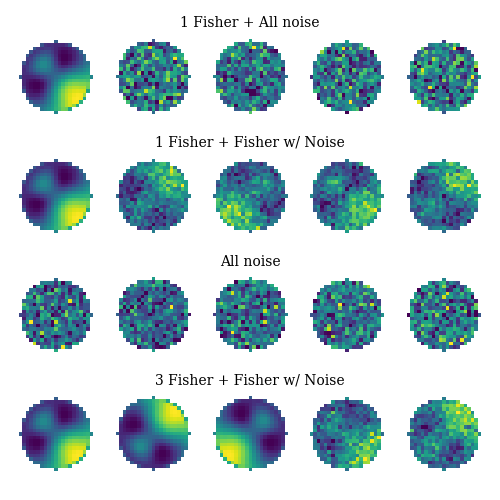}
\end{center}
   \caption{\textbf{Visualization of phase mask initializations.} Each row represents a different initial phase mask sequence.}
\label{fig:masks-init}
\end{figure}

Deep optics for depth perception can be very dependent on the initialization of optical parameters before training~\cite{phasecam}. To find the extent of the effect of mask initialization on performance, we varied the the initial phase masks while keeping number of masks, the optical model, and duration of training fixed. We trained for $200$ epochs. We tested four initializations of sequences of $5$ phase masks as shown in \autoref{fig:masks-init}. The first was uniformly distributed noise from $0$ to $1.2\times 10^{-6}$. The second was the first mask in the sequence set to a Fisher mask while the rest are uniform noise. The third is setting each mask to a rotation of the Fisher mask and adding Gaussian noise parameterized by a $5.35\times 10^{-7}$ mean and $3.05\times 10^{-7}$ standard deviation to $4$ masks. Lastly, we set each mask to a rotation of the Fisher mask and added noise to only the last two masks in the sequence. Of the four initializations, it is clear that the $3$ Fisher masks and $2$ Fisher masks with noise performed the best (\autoref{tab:depth-init}). 

\begin{table}[t]
\begin{center}
\begin{tabular}{|l|c|}
\hline
Initialization & RMSE$\downarrow$ \\
\hline\hline
$1$ Fisher + All noise  & 0.0329 \\
$1$ Fisher + Fisher w/ Noise & 0.0271 \\
All noise  & 0.0254 \\
$3$ Fisher + Fisher w/ Noise & \textbf{0.0207} \\
\hline
\end{tabular}
\end{center}
\caption{\textbf{Quantitative evaluation of phase mask initializations.} Four sequence initializations are evaluated on the monocular depth estimation task. Ultimately, $3$ Fisher masks and $2$ noisy Fisher masks have the best performance after training. }
\label{tab:depth-init}
\end{table}

\begin{figure}[t]
\begin{center}
   \includegraphics[width=0.8\linewidth]{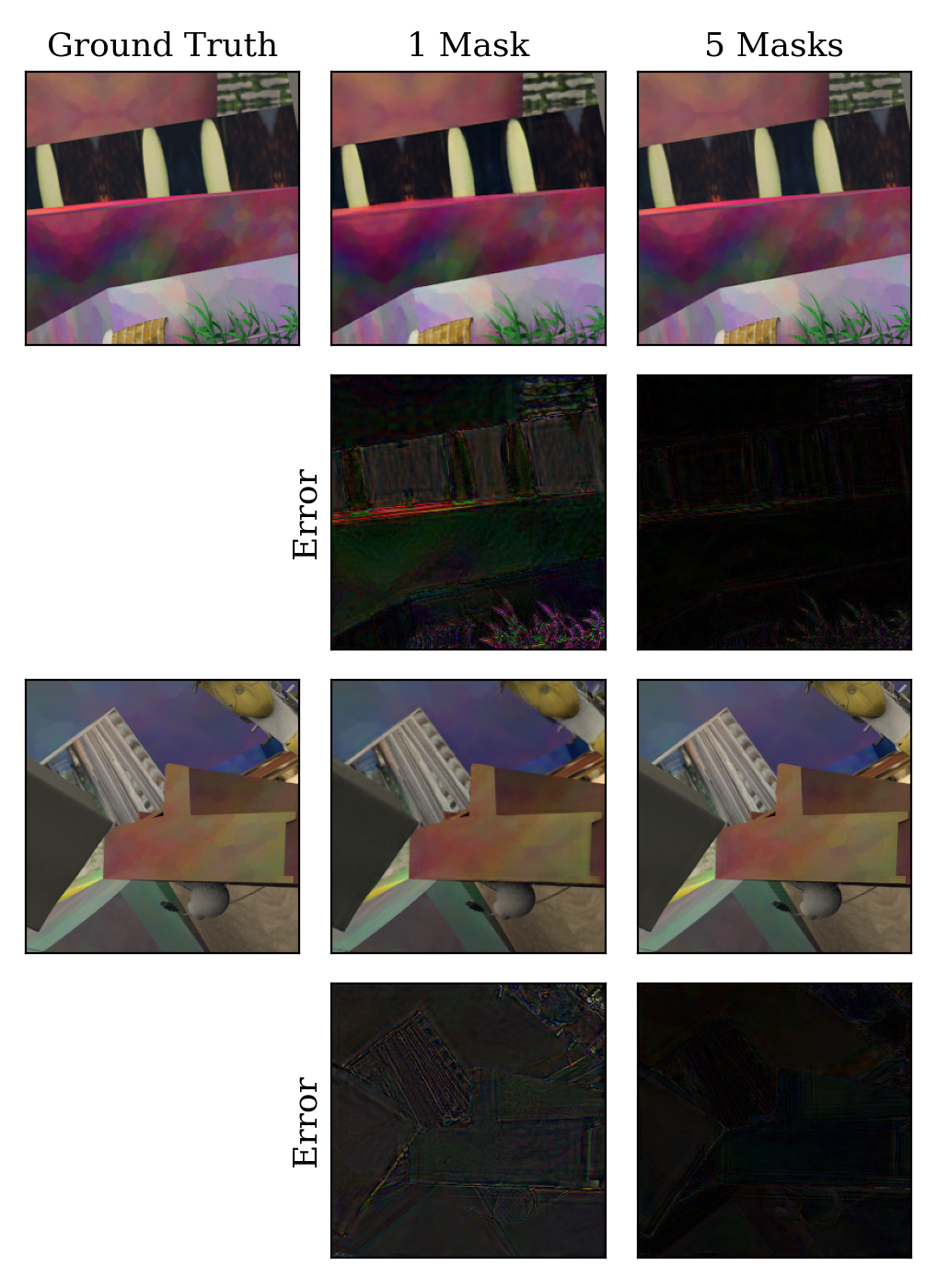}
\end{center}
   \caption{ \textbf{Qualitative results of a specialized network on extended depth-of-field imaging.} Both $1$ and $5$ phase mask systems are evaluated on FlyingThings3D. Error is computed pixel wise between the ground truth all-in-focus image and the reconstructed output and is boosted by a factor of $3$. Notice that the $5$ phase mask system introduces minimal error. }
\label{fig:edof-images}
\end{figure}

\begin{figure}[t]
\begin{center}
   \includegraphics[width=0.9\linewidth]{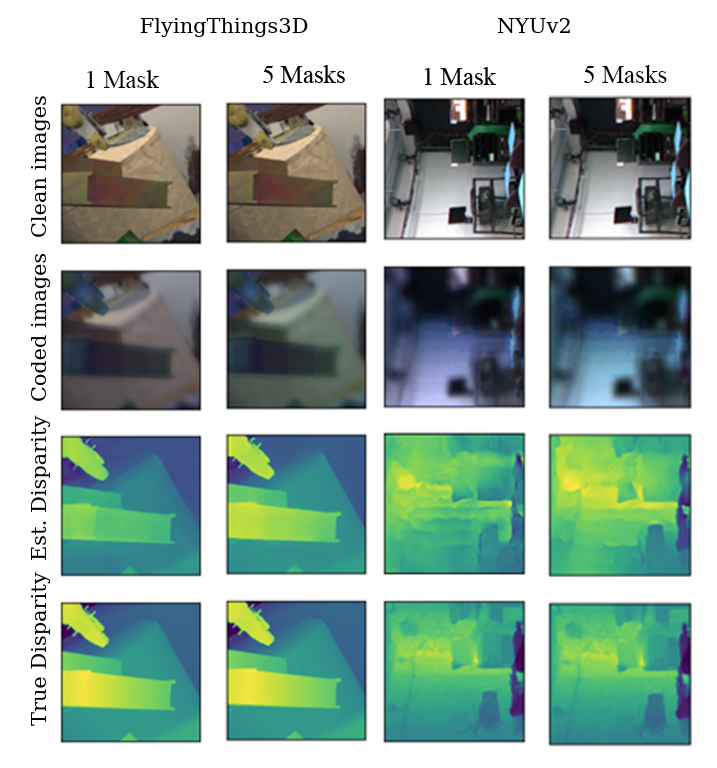}
\end{center}
   \caption{\textbf{Qualitative results of a specialized networks on monocular depth estimation.} Performance using the five phase mask method outperforms one phase mask on both datasets.}
\label{fig:depth-images}
\end{figure}

\begin{figure}[t]
\begin{center}
   \includegraphics[width=0.8\linewidth]{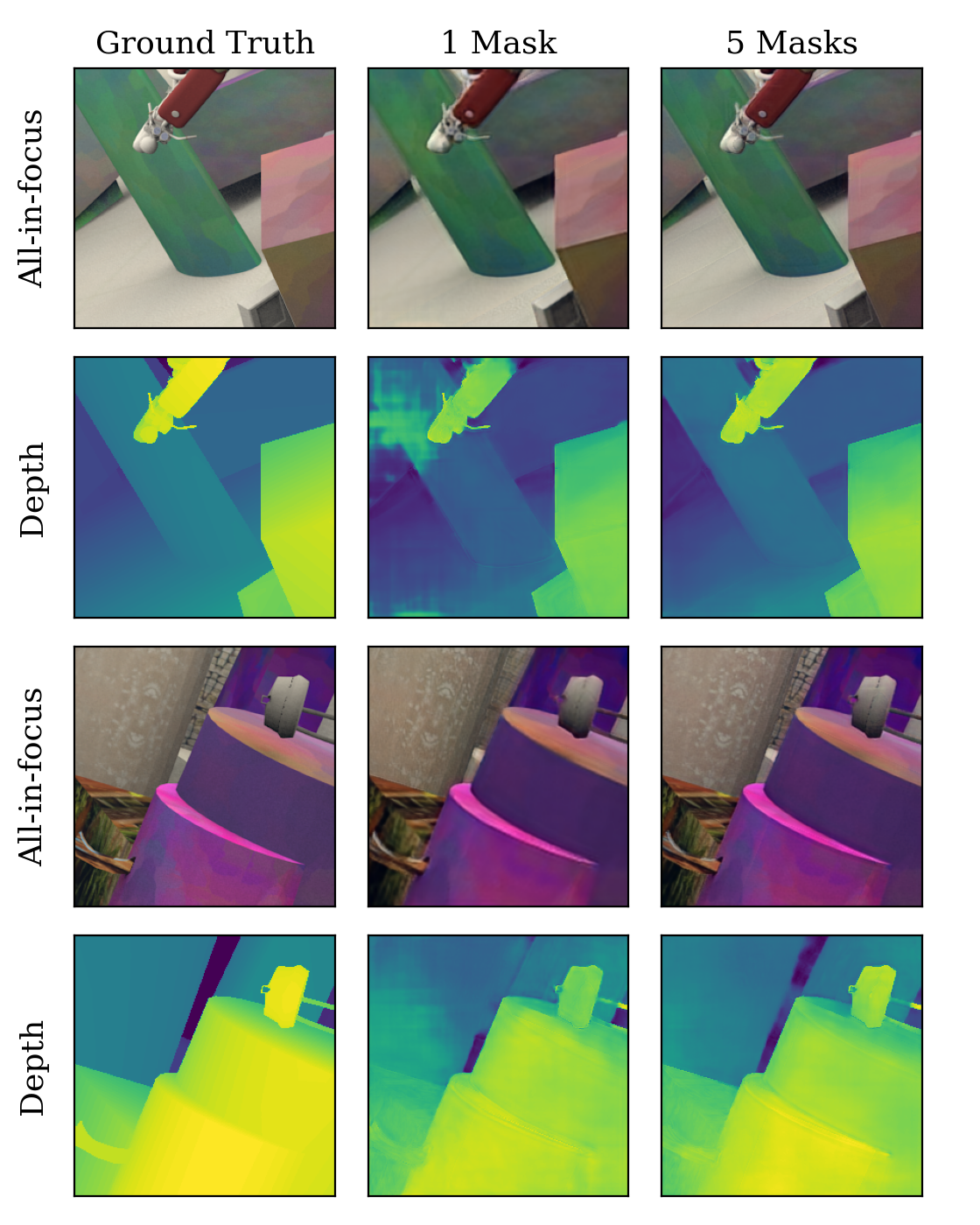}
\end{center}
   \caption{ \textbf{Qualitative results of a joint optimized system for extended depth-of-field imagining and monocular depth estimation.} Both one and five phase mask networks are evaluated on the FlyingThings3D datasets. Notice that five masks has fewer artifacts than a single mask.}
\label{fig:combo-images}
\end{figure}

\subsubsection{Modeling State Switching in SLMs}
Our optical forward model assumes an SLM can swap between two phase patterns instantly. In practice, however, some light will be captured during the intermediate states between phase patterns. These phase patterns, in the worst case, could be random phase patterns, effectively adding noise to our coded images. We model these intermediate states by averaging output images produced by phase masks and the randomized phase patterns weighted by the time that they are displayed for. We model the total exposure time as $100$ms, with various durations of switching times from $1$ to $16$ms per swap. We evaluate our joint optimized network on these new, more noisy, coded images without any additional training (\autoref{fig:switching}). Observe that because the $5$ phase mask system includes more swaps, performance degrades faster than fewer phase mask systems. However, for short switching times, $5$ phase masks still out performs the others without needing any fine tuning.
\begin{figure}[t]
\begin{center}
\includegraphics[width=0.9\linewidth]{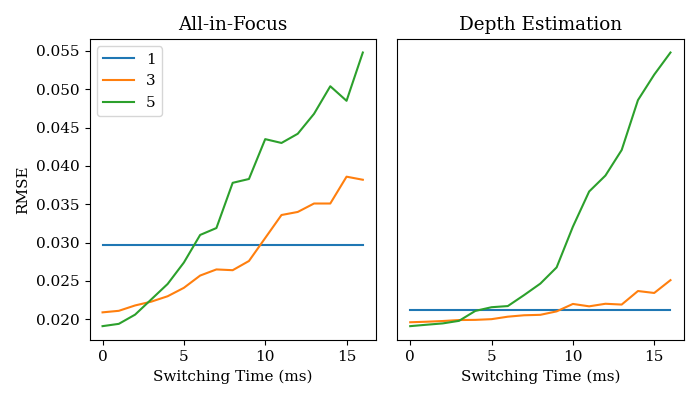}
\end{center}
   \caption{ \textbf{Effect of switching time on joint system performance.} Reconstruction error across phase mask counts as a function of switching time with 100ms overall exposure. Performance of the jointly optimized system degrades as the switching time between phase masks increases, as expected. Our system still performs well when the time spent switching is less than 25\% of the overall exposure. }
\label{fig:switching}
\end{figure}

\section{Results}

We compare our time averaged dynamic PSF method to the state-of-the-art methods for both extended depth-of-field imaging and monocular depth estimation.
The relevant works we compare to are as follows:
\begin{enumerate}
    \item PhaseCam3D \cite{phasecam} used a $23\times 23$ phase mask based on $55$ Zernike coefficients. The phase mask parameters were then end-to-end optimized with a U-Net reconstruction network to perform depth estimation.
    \item Chang et al. \cite{Chang:2019:DeepOptics3D} used a singlet lens introducing chromatic aberrations with radially symmetric PSFs. Similar to \cite{phasecam}, the lens parameters were also then end-to-end optimized.
    \item Ikoma et al. \cite{Ikoma:2021} used a radially symmetric diffractive optical element (DOE). The blurred image was preconditioned with an approximate inverse of the PSF depth dependent blur. The RGB image stack was fed into a U-Net to produce both an all-in-focus image and a depth map. The DOE and U-Net parameters were optimized in an end-to-end fashion.
    \item Liu et al. \cite{Liu:22} used various phase mask parameterizations with the same U-Net architecture as \cite{Ikoma:2021}. One method used pixel-wise height maps (PW) and the other introduced orbital angular momentum (OAM).
    \item Sitzmann et al. \cite{Sitzmann:2018} implements a single DOE based on Zernike coefficients, and solves the Tikhonov-regularized least-squares problem to reconstruct an all-in-focus image.
    \item MiDaS \cite{MIDAS-DPT} and 
ZoeDepth \cite{ZOE} are state of the art single shot monocular depth estimation methods with all-in-focus images as inputs.
\end{enumerate}
Because both \cite{Ikoma:2021} and \cite{Liu:22} simultaneously learn all-in-focus images and depth maps, when comparing against our specialized methods, we take their best performing weighting of each task.

\paragraph{Individual Tasks.}
For monocular depth estimation, our specialized method using a sequence of $5$ phase masks trained for $300$ epochs outperforms prior work on FlyingThings3D (\autoref{tab:depth-sota}). Additionally, our approach performs significantly better and achieves lower error than previous methods on NYUv2 without any additional fine tuning. For extended depth-of-field, our specialized method using a sequence of $5$ phase masks out performs prior work on FlyingThings3D (\autoref{tab:edof-sota}). This demonstrates the benefit of multi-phase mask learning on computational imaging tasks.

\begin{table}[t]
\begin{center}
\begin{tabular}{|l|c|c|}
\hline
Method & FlyingThings3D & NYUv2 \\
\hline\hline
 PhaseCam3D \cite{phasecam}  & 0.521  & 0.382\\
Chang et al. \cite{Chang:2019:DeepOptics3D} & 0.490 & 0.433\\
Ikoma et al. \cite{Ikoma:2021}  & 0.184  &  -\\
\hline
MiDaS \cite{MIDAS-DPT}  & -  &  0.357\\
ZoeDepth \cite{ZOE} & - & 0.277 \\
\hline
TiDy (1) & 0.026 & 0.259 \\
TiDy (5) & \textbf{0.019} & \textbf{0.175} \\
\hline
\end{tabular}
\end{center}
\caption{ \textbf{RMSE comparison of monocular depth estimation methods.} We present quantitative results on two datasets to compare to state of the art optical and single shot monocular depth estimation methods. Our methods performs best with our $5$ phase mask system achieving the lowest error on both datasets. }
\label{tab:depth-sota}
\end{table}

\begin{table}[t]
\begin{center}
\begin{tabular}{|l|c|c|c|}
\hline
Method & RMSE$\downarrow$ & PSNR$\uparrow$ & SSIM$\uparrow$ \\
\hline\hline
Liu et al. \cite{Liu:22} & - & 29.80 & -\\
Ikoma et al. \cite{Ikoma:2021} & 0.1327 & 31.88 & 0.905 \\
Sitzmann et al. \cite{Sitzmann:2018} & - & 32.44 & - \\
\hline
TiDy (1) & 0.0148 & 37.33 & 0.968 \\
TiDy (5) &  \textbf{0.0092} & \textbf{41.11} & \textbf{0.989} \\
\hline
\end{tabular}
\end{center}
\caption{ \textbf{Comparison of extended depth-of-field imaging methods.} We present quantitative results on FlyingThings3D to compare to state-of-the-art. Our methods performs best with our $5$ phase mask system achieving the best PSNR. }
\label{tab:edof-sota}
\end{table}

\paragraph{Multi-Objective Optimization.}
We also evaluate our method against other joint all-in-focus and depth map learning approaches. This problem is challenging because good depth cues to produce depth maps is antithetical to producing an all-in-focus image. Our combined $5$ phase mask trained for $300$ epochs approach outperforms prior jointly trained approaches (\autoref{tab:combo-sota}).

\begin{table}[t]
\begin{center}
\begin{tabular}{|l|c|c|}\hline
        & {All-in-focus} & Depth \\
Method & PSNR$\uparrow$ & RMSE$\downarrow$  \\\hline\hline
Ikoma et al. \cite{Ikoma:2021} & 31.88 & 0.191 \\
Liu et al. \cite{Liu:22} - PW & 29.80 & 0.056\\
Liu et al. \cite{Liu:22} - OAM$_t$ & 25.86 & 0.053\\
\hline
TiDy (1) & 31.20 & 0.052 \\
TiDy (5) & \textbf{34.79} & \textbf{0.034} \\
\hline
\end{tabular}
\end{center}
\caption{ \textbf{Comparison of multi-objective optimization of extended depth-of-field imaging and depth estimation methods.} We compare quantitative results on FlyingThings3D to the state-of-the-art. Our methods performs best with our $5$ phase mask system achieving the best balance between objectives. }
\label{tab:combo-sota}
\end{table}

\section{Limitations}

While we were successful in learning dynamic phase masks to improve state-of-the-art performance on imaging tasks, our method still carries some limitations. First, our optical model assumes perfect switching between phase masks during training. While evaluation with non-zero switching times showed little degradation of performance, accounting for state switching while training could produce phase masks that are more performant. Our optical model also simulates depths as layered masks over an image, which does not account for blending at depth boundaries. Additionally, our method assumes that scenes are static for the duration of a single exposure. Lastly, though their prices are falling, SLMs are still quite expensive and bulky.
\vspace{-5pt}

\section{Conclusion}

This work is founded upon the insight that the set of PSFs that can be described by a single phase mask is non-convex and that, as a result, time-averaged PSFs are fundamentally more expressive.
We demonstrate that one can learn a sequence of phase masks that, when one dynamically switches between them over time, can substantially improve computational imaging performance across a range of tasks, including depth estimation and all-in-focus imaging. 
Our work unlocks an exciting new direction for PSF engineering and computational imaging system design.

\section*{Acknowledgements}
C.M. was supported in part by the AFOSR Young Investigator Program Award FA9550-22-1-0208.

{\small
\bibliographystyle{ieee_fullname}
\bibliography{egbib}
}

\newpage



\section*{Supplementary material}

where $\star$ denotes cross-correlation. Because the Fourier Transform is linear we finally have

$$Y \star Y = \frac{1}{k} \sum_{i=1}^{k} X^{(i)} \star X^{(i)}. \quad (11)$$

Therefore, the convexity of the range of $g$ is equivalent to the convexity of the set $\{X \star X : X \in T_A(N)\}$. We will show the set's projection onto a particular coordinate is not convex.

$$(X \star X)_{s,r} = \sum_{i=0}^{N} \sum_{j=0}^{N} X_{i,j} \overline{X_{i+s,j+r}} \quad (12)$$

where we adopt the convention that $X_{s,r} = 0$ when $s, r > N$ or $s, r < 0$. Observe that cross-correlation can be represented geometrically as shifting $\overline{X}$ over $X$. Let $\mathcal{S}$ be the set of coordinates with non-zero entries in $X$. Applying Lemma 1 to $\mathcal{S}$ shows that $\overline{X}$ and $X$ will overlap at exactly one point. Select points $v, u \in \mathcal{S}$ such that $v - u = \delta$, then,

$$(X \star X)_\delta = X_u \overline{X_v} + \sum_{i=1}^{N^2-1} 0. \quad (13)$$

Because $X_u \overline{X_v} \in \mathbb{T}$, $(X \star X)_\delta \in \mathbb{T}$ which is a non-convex set. Therefore, the set of correlation's of values on the complex unit circle masked by $A$ is also not convex. Consequently, the range of $PSF$ is not a convex set. □

## 2. Discussion: Time Averaging Compared to Multi-Shot Sequences

Our optical model images a static scene through multiple phase masks which we switch between over the course of single exposure (Figure 4a). A natural question, then, is why limit ourselves to a single exposure. Why not capture a burst of images, each with a different phase mask (Figure 4b)?

While it is true that superimposing the outputs of multiple PSFs creates challenges in disambiguating outputs from phase masks, it also offers several benefits. First, because we only capture a single frame, our system uses less memory due to less I/O required. Second, imaging in a single exposure is more light efficient. Over a fixed time interval, a single exposure allows you to capture the entirety of the light from the scene. Multi-shot, alternatively, would miss photons during readout between shots.

## References

[1] Joseph W. Goodman. *Introduction to fourier optics*. Freeman, 2017. 2
[2] Yicheng Wu, Vivek Boominathan, Huaijin Chen, Aswin Sankaranarayanan, and Ashok Veeraraghavan. Phasecam3d

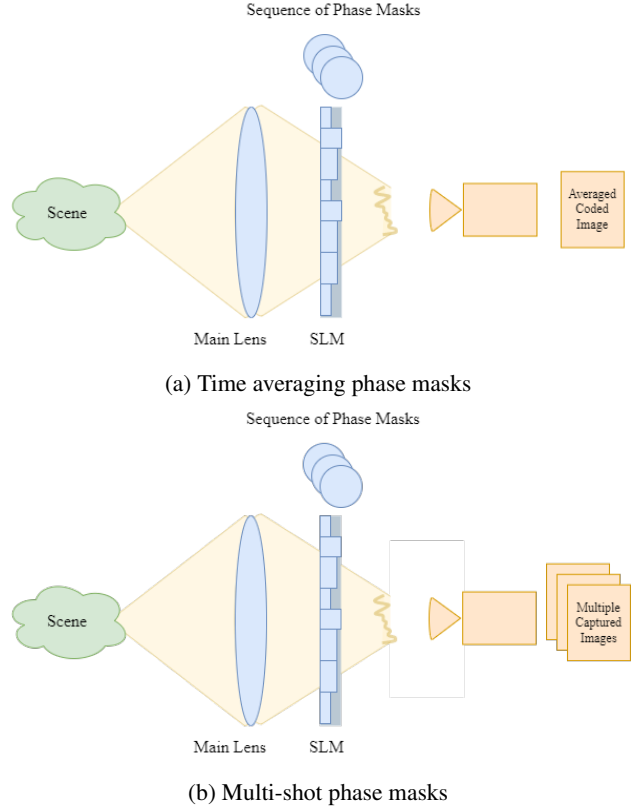

(a) Time averaging phase masks

(b) Multi-shot phase masks

Figure 4: **Time averaging and multi-shot optical systems.** Observe that multi-shot systems capture multiple coded images, while time averaging only captures one. This means our system is more light and memory efficient.

— learning phase masks for passive single view depth estimation. In *2019 IEEE International Conference on Computational Photography (ICCP)*, pages 1–12, 2019. 2



\end{document}